\documentclass[11pt]{article}

\usepackage[preprint]{acl}

\usepackage{times}
\usepackage{latexsym}

\usepackage[T1]{fontenc}

\usepackage[utf8]{inputenc}

\usepackage{microtype}
\usepackage{hyperref}
\usepackage{inconsolata}
\usepackage{amsmath}
\usepackage{amssymb}
\usepackage{amsthm}
\usepackage{graphicx}
\usepackage{booktabs}
\usepackage{colortbl}
\usepackage[normalem]{ulem}

\definecolor{headercol}{HTML}{E8E8E8}    
\definecolor{familycol}{HTML}{EDEDED}    
\definecolor{familytxt}{HTML}{8A8A8A}    
\definecolor{bestcol}{HTML}{DAE6F4}      
\definecolor{secondcol}{HTML}{FCF2D2}    
\definecolor{gaincol}{HTML}{1F7A4C}      
\arrayrulecolor{black}  
 
\newcommand{\gainval}[1]{\textcolor{gaincol}{$+#1$}}        
\newcommand{\best}[1]{\colorbox{bestcol}{\strut\textbf{#1}}}  
\newcommand{\snd}[1]{\colorbox{secondcol}{\strut\underline{#1}}}  
\newcommand{\famlabel}[1]{\textcolor{familytxt}{\textit{#1}}}  

\newtheorem{proposition}{Proposition}

\newif\ifprintcomments

\printcommentstrue

\usepackage{algorithm}      
\usepackage{algorithmic}    
\usepackage{xcolor}     
\usepackage[table]{xcolor}
\usepackage{multirow} 

\newcommand{\methodname}{Hindsight Self-Distillation}
\newcommand{\method}{\textsc{HSD}}

\title{Localizing Credit at the Divergence: Path-Conditioned Self-Distillation for LLM Reasoning}

\author{Yu Li \quad Shu Hong \quad Tian Lan \\
  Department of Electrical and Computer Engineering,
  George Washington University \\
}

\begin{document}
\maketitle


\begin{abstract}
Reinforcement learning from verifiable rewards assigns a single scalar to each rollout, leaving token-level credit assignment underspecified in long reasoning traces. 
On-policy self-distillation addresses this by letting the same model act as a teacher conditioned on privileged information, producing a dense per-token signal. 
But the common choice of a ground-truth answer is only an endpoint cue: on terse-answer tasks, the teacher falls silent at the intermediate positions where path-level guidance matters most.
We propose \methodname{} (\method{}), which conditions the teacher on a successful peer rollout drawn from the current training group. Such a peer is an exact sample from the success-conditioned policy, requiring no additional sampled rollouts. 
By providing a full successful continuation rather than only the final answer,  the resulting credit signal concentrates at the divergence position between a failed rollout and a successful peer.
Across Qwen3-8B and Qwen3-32B on math and code benchmarks, \method{} obtains the best result against GRPO variants and on-policy distillation baselines, with the largest gains on terse-answer tasks such as AIME.
\end{abstract}
\section{Introduction}
\label{sec:intro}
Reinforcement learning from verifiable rewards (RLVR) has become a standard recipe for eliciting reasoning in large language models \citep{guo2025deepseek,yang2025qwen3,li2026right}. For each question, it samples a group of rollouts from the current policy, scores each with a verifier, and updates the policy toward the higher-reward trajectories \citep{shao2024deepseekmath}. 
The simplicity, however, comes with a credit-assignment bottleneck. A long reasoning trace may have a correct setup, one wrong algebraic step, and a final answer that follows from that mistake. Group-relative methods such as GRPO then assign the same scalar advantage to every token in the rollout. As a result, the update can tell that the trajectory failed, but not where it first failed~\citep{li2026oppo}.

On-policy self-distillation offers a way to turn outcome feedback into token-level supervision~\citep{zhao2026self,hubotter2026reinforcement,yang2026self}. 
The model plays two roles simultaneously: a teacher that conditions on the prompt plus a privileged context, typically the ground-truth answer $a^{\!\star}$, and a student that conditions on the prompt alone. 
The per-token KL between the two distributions defines dense credit that is nonzero wherever the answer shifts the teacher's prediction, giving the gradient position-level resolution that the scalar advantage lacks~\citep{song2026survey,li2026rethinking}. 
This makes the key question not whether self-distillation can provide dense gradients, but whether the privileged context contains the information needed to localize mistakes~\citep{yang2026reasoning}.

We examine this on reasoning benchmarks and find that terse answers are starved of the information needed to supervise intermediate reasoning paths~\citep{liu2025learning,li2025inspo,liu2026calibrating}. 
On tasks such as AIME, the answer may be only a three-digit number. Such an endpoint can constrain the final response span, but it does not specify which intermediate continuation would have preserved verifier success. Empirically, we find that answer-conditioned teachers produce almost no predictive shift across most of the chain of thought, with the signal appearing mainly near the final answer. Thus, the teacher is silent exactly where credit assignment matters most~\citep{ding2025micota,deng2024towards}.

The resolution to this bottleneck is already present in the RLVR training loop. For each question, group sampling often produces both failed and successful rollouts from the current policy.
Each successful rollout is an exact verified reasoning path generated by the model's own dynamics, available at zero additional rollout cost.

We propose \methodname{} (\method{}) that shifts privileged-context self-distillation to path-conditioning. Whenever a successful peer exists in the current rollout group, HSD conditions the teacher on that peer trajectory together with the ground-truth answer. When no peer succeeds, it falls back to the answer alone. 
This simple modification transforms the credit assignment landscape.
Conditioning the teacher on this peer trajectory gives the teacher path-level information, not just an endpoint. When a failed rollout and a successful peer share a prefix, the teacher-student discrepancy remains small on the shared prefix and concentrates near the first divergence, yielding a localized signal for where the failed trajectory departed from a successful continuation.

We evaluate \method{} on Qwen3-8B and Qwen3-32B against the leading variants of GRPO and the established on-policy distillation methods, including the hybrid that combines them. \method{} obtains the strongest math and code result on every column at both scales. 
The AIME benchmarks, which the analysis predicts to be hardest, show the largest margins, exceeding five points over the GRPO baseline at every scale and rising past seven points on Qwen3-32B AIME-24. 
Even against the strongest baseline, \method{} retains a several-point lead on math. The empirical credit profile matches the prediction, with near-zero values on the shared prefix and a spike at $\tau$.

\noindent Our key contributions are summarized as follows:
\begin{itemize}
\setlength{\itemsep}{0pt}
\setlength{\parskip}{0pt}
\setlength{\topsep}{-2pt}
\item We examine why endpoint-context self-distillation fails on terse-answer benchmarks: the answer constrains only the terminus of the trajectory, so the teacher goes silent at intermediate positions where path-level guidance is most needed.
\item We propose Hindsight Self-Distillation, a one-line modification that conditions the teacher on a successful peer from the same group. The path-based context localizes credit at the divergence position~$\tau$, and a coverage analysis yields testable predictions about where gains should appear.
\item We report controlled comparisons against seven baselines on Qwen3-8B and Qwen3-32B. \method{} obtains the best math and code result at both scales, with gains concentrated on the terse-answer benchmarks predicted to be hardest.
\end{itemize}
\section{Preliminary}
\label{sec:preliminaries}

\subsection{Related Work}
\label{subsec:background}

\textbf{Reinforcement Learning for LLM Reasoning.} 
Post-training language models with reinforcement learning has become the dominant approach for eliciting reasoning behavior on tasks with verifiable answers \citep{guo2025deepseek, yang2025qwen3}. The standard pipeline samples a group of rollouts from the current policy, scores each by a programmatic verifier, and updates the policy to favor the higher-scoring rollouts within each group. Group-relative policy optimization \citep{shao2024deepseekmath} has become the standard objective in such settings, since the within-group normalization removes the need for a separate value network. GRPO assigns to every token of rollout $i$ the sequence-level advantage
\begin{equation}
A^G_{i,t} \,=\, \frac{R_i - \mu}{\sigma},
\label{eq:grpo}
\end{equation}
where $R_i \in \{0, 1\}$ is the outcome reward and $\mu, \sigma$ are the mean and standard deviation of the rewards within the group. 
The policy is updated through a clipped policy gradient regularized by a reverse-KL penalty against a frozen reference policy. 
Two properties of Eq.~\eqref{eq:grpo} shape the design space we operate in. The advantage is identical for every token of a rollout, so the gradient signal cannot tell which positions in the trajectory were responsible for the outcome. The magnitude of the signal is determined entirely by the spread of rewards within the group, which collapses to zero when all rollouts succeed or all fail~\citep{zhong2026rc,li2026arise}.

\textbf{On-Policy Distillation.} 
Recent work \citep{agarwal2024policy, gu2024minillm,lu2025onpolicydistillation} replaces the sparse sequence-level signal of Eq.~\eqref{eq:grpo} with a dense per-token signal derived from a teacher, retaining the on-policy property that the student trains on its own rollouts while gaining the position-resolved supervision that GRPO lacks.
Variants including OPSD~\citep{zhao2026self} and SDPO~\citep{hubotter2026reinforcement} share a mechanism we refer to throughout as \textit{privileged-context self-distillation}: the policy plays two roles at once, with the student conditioning on the prompt alone and the teacher conditioning on the prompt augmented by a privileged context $c$ drawn from information unavailable at inference. 
The construction makes the teacher's predictions stronger than the student's at positions where $c$ is informative, and the gradient transfers that advantage back into the student without requiring a separately trained reward model~\citep{razin2026makes}.
Existing methods differ only in their choice of $c$. OPSD sets $c$ to the ground-truth answer. SDPO sets $c$ to environment feedback such as a runtime error.
Each option draws on information about the task, and the question of what else $c$ could draw on remains open.

\subsection{Problem Setup}
\label{subsec:setup}
 
The dataset $\mathcal{D} = \{(q, a^{\!\star})\}$ consists of question-answer pairs. The policy $\pi_\theta$ produces a group of $G$ rollouts $\{o_i\}_{i=1}^{G} \sim \pi_\theta(\cdot \mid q)$, with each $o_i = (o_{i,1}, \ldots, o_{i,|o_i|})$ a token sequence and $o_{i,<t}$ its prefix. Each rollout is scored by a verifier into $R_i \in \{0, 1\}$. Every training objective in the paper adds a reverse-KL regularizer $\Omega(\theta) = \beta\,\mathrm{KL}(\pi_\theta \,\|\, \pi_0)$ against a frozen reference policy $\pi_0$, following standard GRPO practice.
 
Privileged-context self-distillation chooses a context $c_i$ for each rollout $i$ and trains $\pi_\theta$ to match the teacher distribution $\pi_\theta(\cdot \mid q, c_i, o_{i,<t})$ on the student's own rollouts. The per-token credit is the log-likelihood ratio
\begin{equation}
A_{i,t} \,=\, \log \frac{\pi_\theta(o_{i,t} \mid q, c_i, o_{i,<t})}{\pi_\theta(o_{i,t} \mid q, o_{i,<t})},
\label{eq:credit}
\end{equation}
which vanishes when $c_i$ fails to shift the teacher's prediction at $t$ and recovers the supervised-learning gradient at the opposite limit where $c_i$ resolves the next token completely. We use the full-vocabulary reverse-KL form of \citet{zhao2026self} throughout, so Eq.~\eqref{eq:credit} appears as the integrand of the teacher-student KL divergence.
 
The setup assumes a verifiable reward signal, ground-truth answers in the training data, and a policy with enough in-context learning capacity to both use $c_i$ as the teacher and ignore it as the student. 
The first two hold by construction on reasoning benchmarks and Section~\ref{sec:experiments} verifies the third across scales.

\vspace{-2mm}
\section{The Privileged-Context Bottleneck}
\label{sec:bottleneck}
\vspace{-2mm}

RLVR delivers a single signal per completed trajectory, while 
on-policy distillation translates that signal into a dense per-token gradient by inserting a privileged context into the teacher's prompt.
This section asks what the ideal per-token credit looks like, shows that the standard answer-conditioned teacher fails to deliver it on terse-answer benchmarks, and identifies a source of path-level information already present in the training loop.

\vspace{-2mm}
\subsection{The Ideal Token-Level Teacher}
\label{subsec:ideal}
 \vspace{-2mm}
 
Given a prompt $q$ and the current rollout policy $\pi = \pi_\theta$,
let $R(o) \in \{0,1\}$ be the verifier outcome for a complete trajectory~$o$, and assume the prompt has nonzero success probability
\begin{equation}
  Z_\pi(q) \triangleq \Pr_{O \sim \pi(\cdot \mid q)}\!\big[R(O)=1\big] > 0.
\label{eq:success_prob}
\end{equation}
The verifier induces a success-conditioned trajectory distribution
\begin{equation}
  \pi^+(o \mid q)
  \;\triangleq\;
  \pi(o \mid q,\, R(o)=1)
  \;=\;
  \frac{\pi(o \mid q)\, R(o)}{Z_\pi(q)}.
\label{eq:success_dist}
\end{equation}
A teacher that supervised the student against $\pi^{+}$ directly would push the student to imitate its own successful behavior while suppressing failed trajectories.
The next-token form of $\pi^{+}$ is exact and reduces to a single application of Bayes' rule.

\begin{proposition}
[Verifier-conditioned token transform]
\label{prop:transform}
Let $h_t = (q, o_{<t})$ be a reachable prefix and define the \emph{success value}
\begin{equation}
  V_\pi(h_t) \;\triangleq\; \Pr_{O \sim \pi}\!\big[R(O)=1 \mid h_t\big].
\label{eq:value}
\end{equation}
For any next token~$y$ with $\pi(y \mid h_t) > 0$ and $V_\pi(h_t) > 0$,
\begin{equation}
  \pi^+(y \mid h_t)
  \;=\;
  \pi(y \mid h_t)\,
  \frac{V_\pi(h_t\, y)}{V_\pi(h_t)}.
\label{eq:ideal_token}
\end{equation}
Here $h_t\,y$ denotes the prefix~$h_t$ extended by one token~$y$, so $V_\pi(h_t\,y) = \Pr_{O \sim \pi}\big[R(O)=1 \mid h_t, O_t = y\big]$ is the success probability after committing to~$y$ at position~$t$.
\end{proposition}

Equation~\eqref{eq:ideal_token} names the ideal per-token credit: a token should be favored exactly when it raises the policy's probability of eventual verifier success.
The construction is the success-conditioning view of RL as inference~\citep{levine2018reinforcement}, and the multiplicative reweighting matches the twist used in twisted sequential Monte Carlo over trajectories~\citep{naesseth2015nested}.
An immediate consequence connects the ideal credit to the token-level advantage.

\begin{figure}[t]
\centering
\includegraphics[width=\columnwidth]{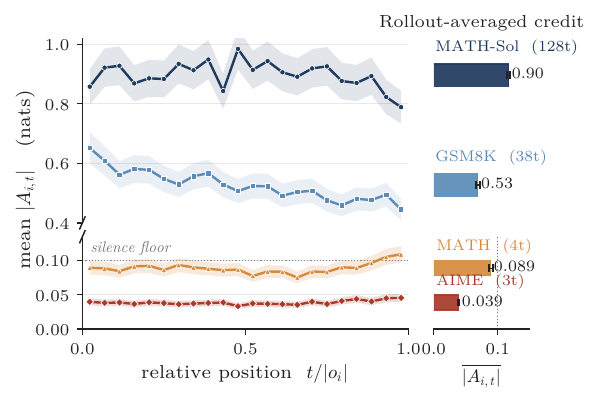}
\caption{Mean sampled-token predictive shift $|A^{(c)}_{i,t}|$ as a function of relative position $t/|o_i|$, averaged over $1{,}024$ Qwen3-8B rollouts per benchmark, with $\pm 1$~s.e.\ bands. Reference-context length is shown in parentheses. \textbf{Top:} verbose-reference benchmarks where $c_i$ is the full worked solution. \textbf{Bottom:} terse-target benchmarks where $c_i$ is the bare answer. The right marginal shows rollout-averaged shift per benchmark.}
\label{fig:shift}
\end{figure}

\subsection{The Silence of the Answer-Conditioned Teacher}
\label{subsec:diagnostic}
\vspace{-1mm}
The standard choice sets the privileged context to the ground-truth answer: $c = a^{\!\star}$.
For the answer-conditioned teacher $T_{\mathrm{ans}}(y \mid h_t) \triangleq \pi_\theta(y \mid q, a^{\!\star}, o_{<t})$ to reproduce the ideal credit in Eq.~\eqref{eq:success_dist}, the in-context substitution $c = a^{\!\star}$ would have to reconstruct the ratio $V_\pi(h_t\,y) / V_\pi(h_t)$ at every reachable prefix and candidate token.
The right-hand side asks for the full landscape of trajectories reachable from $h_t$ under the current policy.
Knowing the endpoint identifies which terminal state should be reached but does not by itself expose path-dependent success values at intermediate positions, so the answer-conditioned teacher recovers the ideal only at positions where the next token is constrained directly by the endpoint, essentially the final-answer span.

We measure the answer-conditioned teacher's realized credit through the sampled-token log-ratio $A^{(c)}_{i,t} = \log T_{\mathrm{ans}}(o_{i,t}\mid h_t) - \log \pi_\theta(o_{i,t}\mid h_t)$ across four benchmarks that vary the context while holding the task fixed at mathematical reasoning.
The verbose-reference settings sustain rollout-averaged shifts of $\overline{|A|} = 0.90$ and $0.53$ nats respectively. 
The bare-answer settings collapse to $\overline{|A|} = 0.089$ on MATH-numeric and $0.039$ on AIME, a $23\times$ drop monotone in context length.
The bottom panel of Figure~\ref{fig:shift} shows the per-position curves for the bare-answer regime, with the visible structure being an upward drift near the final-answer span, the one region where the endpoint constrains the teacher's prediction directly.

\vspace{-1mm}
\begin{figure}[!t]
\centering
\includegraphics[width=\columnwidth]{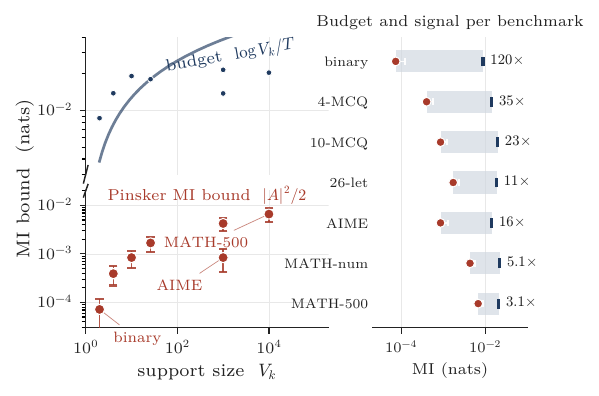}
\caption{Sampled-token predictive shift $\overline{|A|}$ across seven benchmarks, with answer support sizes ranging from $V_k = 2$ (binary) to $V_k \approx 10^4$. The squared shift $|A|^{2}/2$ serves as a proxy for the conditional information the context carries at each position; the support-size-scaled reference $\log V_k / T$ in the top panel sits one to two orders of magnitude above it. The shift shrinks monotonically as the answer becomes more compact, with the gap (right) widening on terse-answer settings.}
\label{fig:ceiling}
\label{fig:ceiling}
\end{figure}
\vspace{-1mm}

\paragraph{An information-theoretic perspective.}
The silence is not an artifact of model capacity but a structural property of the answer itself.

\vspace{-1mm}
\begin{proposition}[Endpoint-context information ceiling]
\label{prop:ceiling}
Let $a^{\!\star}$ be a random variable supported on at most $V_k$ distinct values, jointly distributed with the rollout $o = (o_1, \ldots, o_T)$ conditional on~$q$. Then
\begin{equation}
\sum_{t=1}^{T} I\!\left(o_t;\, a^{\!\star} \,\middle|\, q,\, o_{<t}\right) \;\leq\; \log V_k.
\label{eq:ceiling}
\end{equation}
\end{proposition}
\vspace{-1mm}

\vspace{-1mm}
\begin{proof}
The chain rule of conditional mutual information gives
$\sum_{t} I(o_t; a^{\!\star} \mid q, o_{<t}) = I(o_{1:T}; a^{\!\star} \mid q) \leq H(a^{\!\star} \mid q) \leq \log V_k.$
\end{proof}
\vspace{-1mm}

For AIME ($V_k \approx 10^3$), this amounts to roughly $6.9$ nats across ${\sim}500$ positions. No answer-conditioned teacher, regardless of capacity, can extract more from $a^{\!\star}$ than the answer contains.
The bound corroborates the empirical pattern: the realized per-token shift sits one to two orders of magnitude below the amortized per-position budget $\log V_k / T$, and the gap grows monotonically as the answer becomes more compact (Figure~\ref{fig:ceiling}).

\vspace{-2mm}
\subsection{Successful Rollouts as Path-Level Witnesses}
\label{subsec:why}
 \vspace{-2mm}
If endpoint contexts fail because they constrain only the terminus of a trajectory, an effective privileged context must constrain the path. 
Worked reasoning paths are usually scarce in supervised data, but the group-relative training loop already produces them as a byproduct of group sampling.

\textbf{On-policy witnesses.} 
A group of $G$ rollouts $\{o_j\}_{j=1}^{G}$ is sampled i.i.d.\ from $\pi(\cdot \mid q)$ and scored into $\{R_j\}_{j=1}^{G}$. Filtering by $R_j = 1$ is exactly rejection sampling from $\pi^{+}$:
\begin{equation}
\begin{aligned}
\Pr[O = o_j \mid R(o_j) = 1,\, q]
\;&=\; \frac{\pi(o_j \mid q)\, R(o_j)}{Z_\pi(q)} \\
\;&=\; \pi^{+}(o_j \mid q).
\end{aligned}
\label{eq:rejection}
\end{equation}
Any verifier-accepted rollout $o^{+}$ from the current group is a fresh sample from the success-conditioned policy under $\pi_\theta$, obtained at zero additional rollout cost.
 
\textbf{From distribution to demonstration.}
The single sample $o^{+}$ enters the teacher's input as a worked path from $q$ to $a^{\!\star}$ that the student attends to and copies from. 
The construction does not estimate $V_\pi(h_t\,y) / V_\pi(h_t)$ explicitly; whether the resulting in-context teacher distribution actually approximates $\pi^{+}$ is an empirical question we settle in Figure~\ref{fig:credit-profile}. Figure~\ref{fig:teaser} contrasts the two regimes. Endpoint context produces a near-flat per-token shift on a failing rollout (top); path context $c = (a^{\!\star}, o^{+})$ yields the prefix-floor, spike-at-$\tau$, decay-beyond profile predicted by the construction (bottom). 
The construction breaks down at the boundaries $Z_\pi(q) \in \{0, 1\}$, where the group has either no successful rollout to witness or no failing rollout to supervise. Section~\ref{sec:method} formalizes the construction and the fallback.

\vspace{-1mm}
\begin{figure}[!t]
\centering
\includegraphics[width=\columnwidth]{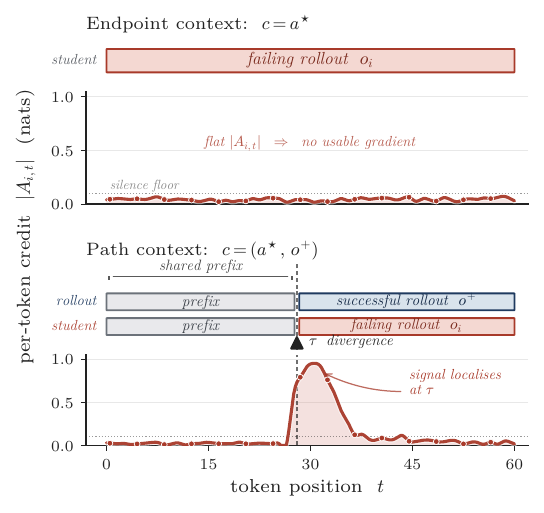}
\caption{Endpoint context versus path context. \textbf{Top:} conditioning the teacher on $a^{\!\star}$ alone produces a near-flat per-token shift on a failing rollout. \textbf{Bottom:} conditioning on $c=(a^{\!\star}, o^{+})$ with $o^{+}$ a successful peer from the same group. The two trajectories share a prefix up to a divergence position $\tau$, the per-token shift is near zero on the shared prefix, spikes at $\tau$, and decays beyond.}
\label{fig:teaser}
\end{figure}
\vspace{-1mm}

\section{Hindsight Self-Distillation}
\label{sec:method}

\subsection{Relabeling the Context in Hindsight}
\label{subsec:hindsight}

Consider a group of $G$ rollouts $\{o_j\}_{j=1}^{G}$ for question $q$ with binary verifier rewards $\{R_j\}_{j=1}^{G}$. For rollout $i$, let $\mathcal{S}_i = \{j \neq i : R_j = 1 \}$ denote its successful peers, and let $o^{+}_i$ denote one such peer drawn uniformly from $\mathcal{S}_i$ when $\mathcal{S}_i \neq \emptyset$. \method{} sets the privileged context to
\begin{equation}
c_i \;=\;
\begin{cases}
(a^{\!\star},\; o^{+}_i), & \mathcal{S}_i \neq \emptyset, \\[2pt]
a^{\!\star}, & \mathcal{S}_i = \emptyset,
\end{cases}
\label{eq:context}
\end{equation}
falling back to the endpoint context on groups where no rollout has yet succeeded.
Specifically, the peer $o^{+}_i$ excludes $i$ itself to keep the teacher's and student's contexts distinct.
The pair $(a^{\!\star}, o^{+}_i)$ keeps $a^{\!\star}$ explicit even though it is a deterministic function of $o^{+}_i$, as explicit inclusion stabilizes the teacher's tail behavior past the divergence position. 
 
Given $c_i$, \method{} adopts the per-token reverse-KL objective of \citet{zhao2026self}. Writing $\pi^{c}_{i,t}$ and $\pi_{i,t}$ for the teacher distribution $\pi_\theta(\cdot \mid q, c_i, o_{i,<t})$ and the student distribution $\pi_\theta(\cdot \mid q, o_{i,<t})$, the objective is
\begin{equation}
\mathcal{J}(\theta) \;=\; \mathbb{E}
\!\left[
\frac{1}{G} \sum_{i=1}^{G} \frac{1}{|o_i|} \sum_{t=1}^{|o_i|}
\mathrm{KL}\!\big(\, \pi^{c}_{i,t} \,\big\|\, \pi_{i,t} \,\big)
\right],
\label{eq:objective}
\end{equation}
where the teacher distribution carries a stop-gradient and the full loss adds the reverse-KL penalty $\beta\,\mathrm{KL}(\pi_\theta \,\|\, \pi_0)$ against the reference policy $\pi_0$.
 
Each training step performs two forward passes per rollout: a student pass scoring $o_i$ under $\pi_{i,t}$, and a teacher pass scoring the same $o_i$ under $\pi^{c}_{i,t}$ with $c_i$ prepended. The teacher pass is longer than the student pass by $|c_i|$ tokens, so \method{} adds at most one demonstration-length forward pass per rollout relative to endpoint-context distillation, and draws no additional rollouts beyond the $G$ that GRPO already produces.


\subsection{Demonstrations Localize Credit at the Divergence Position}
\label{subsec:analysis}

The construction in Eq.~\eqref{eq:context} induces a sharp signature in the per-token credit. Take a failing rollout $o_i$ ($R_i = 0$) and a successful peer $o^{+}_i$ from the same group, and let
\begin{equation}
\tau \;\triangleq\; \min \{\, t \geq 1 : o_{i,t} \neq o^{+}_{i,t} \,\}
\label{eq:divergence}
\end{equation}
denote the position at which the two trajectories first disagree, so that $o_{i,<\tau} = o^{+}_{i,<\tau}$. The per-token credit of Eq.~\eqref{eq:credit} traces through three regimes.

\textbf{Pre-divergence ($t < \tau$).}
The teacher and student read the same prefix tokens at position~$t$, while the teacher additionally has $(a^{\!\star}, o^{+}_i)$ in context.
Since the demonstrated continuation past $\tau$ is consistent with the shared prefix but does not select between alternative continuations the student would otherwise consider, the credit $A_{i,t}$ stays near the floor.

\textbf{At divergence ($t = \tau$).}
The prefix still agrees, and the credit reduces to $A_{i,\tau}$ from~\eqref{eq:credit}.
The teacher sees the demonstrated continuation $o^{+}_{i,\tau}$ past the shared prefix and concentrates mass on it, while the student spreads mass across continuations including the failing token $o_{i,\tau} \neq o^{+}_{i,\tau}$.
The resulting credit is large in magnitude, with sign indicating that the student should have followed $o^{+}_{i,\tau}$.

\textbf{Post-divergence ($t > \tau$).}
The failing rollout has now departed from the demonstration, and the student's prefix $o_{i,<t}$ no longer matches the demonstration's prefix. The teacher's in-context guidance weakens as the demonstrated continuation drifts further from the student's actual prefix. 
The post-$\tau$ credit decays over the following $10$ to $15$ tokens before returning to the prefix floor, as measured in Figure~\ref{fig:credit-profile}.

The three-regime structure shows where the credit concentrates.  
It remains to ask whether it concentrates in the right direction.
By the rejection sampling identity (Eq.~\ref{eq:rejection}), the peer 
$o^{+}_i$ is drawn from $\pi^{+}(\cdot \mid q)$. Conditional on sharing 
the prefix $h_\tau = (q, o_{i,<\tau})$, the autoregressive factorization 
gives $o^{+}_{i,\tau} \sim \pi^{+}(\cdot \mid h_\tau)$. Each training 
step therefore pushes $\pi_\theta$ toward a single token sampled from 
the ideal teacher at position~$\tau$, and over independent peer draws 
the averaged gradient equals the cross-entropy gradient toward 
$\pi^{+}(\cdot \mid h_\tau)$. Hence, HSD rules out the failure mode where credit 
systematically favors tokens of zero ideal mass.

\textbf{Coverage.} 
\method{} provides path supervision only when a successful peer exists.
Let $p$ be the per-question success rate. The fraction of rollouts that both fail and have a successful peer is
\begin{equation}
  f(p, G) \;=\; (1-p)\bigl(1 - (1-p)^{G-1}\bigr),
\label{eq:fpg}
\end{equation}
which is zero at $p=0$ and $p=1$ and peaks at $p^{\star} = 1 - G^{-1/(G-1)}$.
For $G = 8$, $p^{\star} \approx 0.25$. Coverage vanishes at $p = 0$ (no successful peer to draw from) and at $p = 1$ (no failing rollout to supervise). 
Later in Section~\ref{sec:experiments}, we will show that the AIME benchmarks at $p \approx 0.22$ on Qwen3-8B sit near $p^{\star}$ with the largest gains, while HumanEval+ at $p \approx 0.8$ with minimal gains.
\section{Experiments}
\label{sec:experiments}

\begin{table*}[!t]
\centering
\caption{Pass@1 on five held-out reasoning benchmarks for Qwen3-8B and Qwen3-32B. All methods are trained on the same data with group size $G=8$ for $1{,}500$ steps. The best result per column is shown in \colorbox{bestcol}{\strut light blue}, the second-best in \colorbox{secondcol}{\strut light yellow}. The bottom two rows report the gain of \method{} over the baselines.}
\label{tab:main}
\small
\setlength{\tabcolsep}{3.8pt}
\renewcommand{\arraystretch}{1.20}
\begin{tabular}{l c c c c c | c c c c c}
\toprule
 & \multicolumn{5}{c|}{\textbf{Qwen3-8B-Base}}
 & \multicolumn{5}{c}{\textbf{Qwen3-32B-Base}} \\
\cmidrule(lr){2-6} \cmidrule(lr){7-11}
\textbf{Method}
 & MATH-500 & AIME-24 & AIME-25 & LCB-v5 & HE+
 & MATH-500 & AIME-24 & AIME-25 & LCB-v5 & HE+ \\
\midrule
\rowcolor{familycol}
\multicolumn{11}{l}{\famlabel{GRPO Variants}} \\
GRPO          & 78.4 & 24.5 & 21.6 & 32.0 & 80.5 & 85.2 & 38.7 & 34.5 & 41.6 & 86.3 \\
Dr.\,GRPO     & 79.2 & 25.6 & 22.5 & 32.6 & 80.7 & 85.9 & 39.8 & 35.4 & 42.2 & 86.5 \\
GSPO          & 79.7 & 26.3 & 23.2 & 33.0 & 80.9 & 86.4 & 40.5 & 36.0 & 42.6 & 86.6 \\
DAPO          & 80.1 & 26.8 & 23.7 & 33.4 & 81.1 & 86.8 & 41.1 & 36.7 & 42.9 & 86.8 \\
\addlinespace[6pt]
\rowcolor{familycol}
\multicolumn{11}{l}{\famlabel{OPD Methods}} \\
OPSD          & 80.9 & 27.3 & 24.4 & 33.7 & 81.4 & 87.0 & 41.5 & 37.1 & 43.0 & 86.9 \\
SDPO          & 81.4 & 27.9 & 25.0 & 34.5 & \best{82.0} & 87.5 & 42.3 & 37.8 & 43.7 & \best{87.4} \\
RLSD          & 81.6 & 28.2 & 25.3 & 34.8 & 81.6 & 87.7 & 42.8 & 38.2 & 43.9 & 87.0 \\
GRPO\,+\,OPSD & \snd{81.9} & \snd{28.6} & \snd{25.7} & \snd{35.1} & 81.7 & \snd{87.9} & \snd{43.2} & \snd{38.6} & \snd{44.1} & 87.1 \\
\textbf{\method{} (ours)}
              & \best{83.6} & \best{31.4} & \best{28.1} & \best{37.2} & \snd{81.9}
              & \best{89.4} & \best{46.0} & \best{41.7} & \best{46.8} & \snd{87.3} \\
\midrule
$\Delta$ over GRPO
              & \gainval{5.2} & \gainval{6.9} & \gainval{6.5} & \gainval{5.2} & \gainval{1.4}
              & \gainval{4.2} & \gainval{7.3} & \gainval{7.2} & \gainval{5.2} & \gainval{1.0} \\
$\Delta$ over OPD\,best
              & \gainval{1.7} & \gainval{2.8} & \gainval{2.4} & \gainval{2.1} & \textcolor{gaincol}{$-0.1$}
              & \gainval{1.5} & \gainval{2.8} & \gainval{3.1} & \gainval{2.7} & \textcolor{gaincol}{$-0.1$} \\
\bottomrule
\end{tabular}
\end{table*}

\subsection{Experimental Setup}
\label{subsec:setup}

We train Qwen3-8B and Qwen3-32B \citep{yang2025qwen3} on a mixture of mathematical reasoning and code generation. The math mixture combines MATH \citep{hendrycks2021measuring}, GSM8K \citep{cobbe2021training}, and Numina-Math \citep{li2024numinamath} for $94\mathrm{k}$ training problems; the code mixture combines APPS \citep{hendrycks2021measuring2} and the LiveCodeBench training pool~\citep{jain2025livecodebench} for $19\mathrm{k}$ problems. 
Verifiers are symbolic equivalence for math and unit-test execution for code. All methods share the same outer loop: sample a question, draw $G=8$ rollouts at temperature $1.0$, score with the verifier, apply the method-specific update. We train for $1{,}500$ steps at learning rate $1{\times}10^{-6}$, batch $128$, AdamW with $\beta_1{=}0.9$, $\beta_2{=}0.95$, and reference-policy KL coefficient $\beta=0.001$. Maximum rollout length is $4{,}096$ tokens for math and $8{,}192$ for code.

We report pass@1 on MATH-500~\citep{hendrycks2021measuring}, AIME~2024 \& 2025\footnote{\url{https://huggingface.co/datasets/AI-MO/aimo-validation-aime}}, LiveCodeBench-v5~\citep{jain2025livecodebench}, and HumanEval+~\citep{liu2023your}, averaged over $16$ samples at temperature $1.0$ except for HumanEval+ where we use greedy decoding. 
The GRPO family of baselines includes GRPO~\citep{shao2024deepseekmath}, Dr.\,GRPO~\citep{liu2025understanding}, GSPO~\citep{zheng2025group}, and DAPO~\citep{yu2026dapo}. 
The on-policy distillation family includes OPSD~\citep{zhao2026self}, SDPO \citep{hubotter2026reinforcement}, and RLSD \citep{yang2026self}, together with the hybrid GRPO\,+\,OPSD that linearly combines the GRPO advantage and the OPSD distillation loss at mixing coefficient $0.5$. 
SDFT \citep{shenfeld2026self} targets continual learning rather than reinforcement learning from verifiable rewards and appears separately in Appendix~\ref{app:sdft}.
All baselines share the data, optimizer, and reference policy; differences trace to the supervision signal alone.

\subsection{Main Result}
\label{subsec:main}

Table~\ref{tab:main} reports pass@1 on five held-out benchmarks at the two scales. \method{} obtains the best math and code result on every column at both scales, and the second-best on HumanEval+, where SDPO leads by a tenth of a point.

Two patterns in the table are worth singling out. The first is where the gain over the OPD family sits: GRPO\,+\,OPSD is the strongest non-\method{} method on math and code because it combines sparse-reward RL with dense-distillation supervision, and the residual gap to \method{} is what path-based supervision adds beyond that linear combination. The argument of Section~\ref{sec:bottleneck} predicts the gap directly, since any linear combination of two signals bounded by the $\log V_k$ ceiling remains bounded, while a path-based context lifts the ceiling. The second pattern is where the gain over GRPO concentrates: the AIME benchmarks, with terse numerical answers, carry the largest values in the $\Delta$ row at both scales, exactly the regime Proposition~\ref{prop:ceiling} predicts to be most starved of signal under endpoint contexts. HumanEval+ saturates across the board, with SDPO marginally ahead and the $\Delta$ row registering its smallest entries there.

\subsection{Ablation Study and Analysis}
\label{subsec:ablation}

Figure~\ref{fig:analysis} consolidates four diagnostic views of \method{} on Qwen3-8B, and Figure~\ref{fig:credit-profile} reports the per-token credit profile around the divergence position $\tau$. The two figures verify the analytical claims of Section~\ref{sec:method} and pin down which design choices are load-bearing.

\textbf{Training dynamics and scaling.} Panel (a) of Figure~\ref{fig:analysis} tracks AIME-25 accuracy through training. \method{} separates from the baselines within the first few hundred steps and the gap widens monotonically through step $1{,}500$. Panel (b) sweeps group size $G$ jointly with step count: AIME-25 improves with $G$ at every horizon, with most of the gain captured by $G=8$. The mechanism appears in panel (c). Coverage of path supervision, measured as $\Pr[\,|\mathcal{S}_i|>0\,]$, is already near $80\%$ on easy and medium problems at $G=4$, so further increases in $G$ contribute most on the hard tail, where coverage rises from below $20\%$ at $G=4$ to roughly two-thirds at $G=16$.

\textbf{Choice of demonstration.} Panel (d) of Figure~\ref{fig:analysis} compares four rules for drawing $o^{+}_i$ from $\mathcal{S}_i$: shortest length, highest perplexity under $\pi_\theta$, lowest perplexity, and uniform sampling. Lowest-perplexity selection edges uniform sampling by a fraction of a point, and uniform sampling sits comfortably above the other two rules. Highest-perplexity selection is materially worse, indicating that peers whose own log-likelihood under the current policy is low carry diluted signal. We use uniform sampling throughout because the gain from lowest-perplexity selection does not justify the extra bookkeeping, and because the localization argument of Section~\ref{subsec:analysis} requires only the path structure rather than a particular choice within it.

\textbf{Endpoint fallback.} Removing the fallback to $c_i = a^{\!\star}$ when $\mathcal{S}_i = \emptyset$ degrades AIME-25 and MATH-500 measurably at $8$B. The fallback engages during the first few hundred steps on the hardest problems before the policy produces its first successes, and without it those problems receive no gradient at all during the early phase.

\begin{figure}[t]
\centering
\includegraphics[width=\columnwidth]{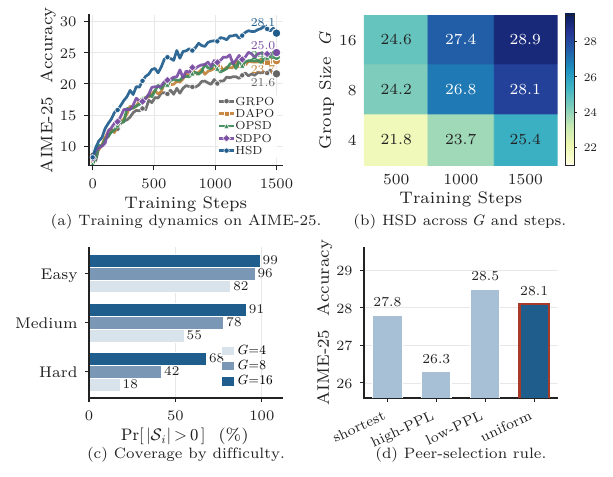}
\caption{Diagnostic views of \method{} on Qwen3-8B. \textbf{(a)} Training dynamics on AIME-25 over $1{,}500$ steps for \method{} against four baselines; final values are reported inline at the right edge. \textbf{(b)} \method{} as a function of group size $G$ and training-step count, accuracy reported per cell. \textbf{(c)} Path-supervision coverage $\Pr[\,|\mathcal{S}_i|>0\,]$ stratified by problem difficulty and group size; coverage on hard problems climbs sharply with $G$. \textbf{(d)} Peer-selection rule ablation: uniform sampling (ours) is within evaluation noise of the best alternative.}
\label{fig:analysis}
\end{figure}

\begin{figure}[t]
\centering
\includegraphics[width=\columnwidth]{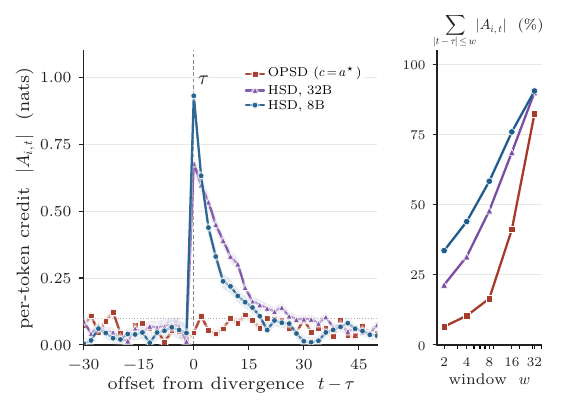}
\caption{Empirical per-token credit $|A_{i,t}|$ as a function of offset $t - \tau$ from the divergence position, averaged over $512$ failing rollouts. \textbf{Left:} \method{}'s signal localizes at $\tau$ at both scales, with the $32$B profile carrying a broader tail than $8$B; the OPSD baseline conditioned on $c = a^{\!\star}$ remains in the silence band across the full window. \textbf{Right:} cumulative fraction of credit mass within $|t - \tau| \leq w$, showing that \method{} concentrates roughly $55\%$ of its mass within $\pm 8$ tokens of $\tau$ while OPSD spreads its mass uniformly.}
\label{fig:credit-profile}
\end{figure}

\textbf{Per-token credit profile.} Figure~\ref{fig:credit-profile} aligns $512$ failing rollouts at the divergence position $\tau$ and reports the measured per-token credit. \method{}'s signal at both scales remains in the silence band over the shared prefix, peaks sharply at $\tau$, and decays back to the prefix baseline over the following $15$ to $20$ tokens, matching the prediction of Section~\ref{subsec:analysis}. The OPSD baseline conditioned on $c=a^{\!\star}$ stays in the silence band over the full window with no peak at $\tau$ and no decay structure, consistent with Proposition~\ref{prop:ceiling}. The $32$B profile shows a broader and more irregular tail than the $8$B profile, yet the right panel reports that roughly $55\%$ of \method{}'s credit mass concentrates within $\pm 8$ tokens of $\tau$ at both scales, against under one-fifth for OPSD. The localization is a structural property of the construction rather than an artifact of model scale.

\textbf{Compute cost.} Each \method{} step performs the same $G$ rollouts as GRPO plus two forward passes per rollout, of which the teacher pass is longer than the student pass by $|c_i|$ tokens. On our math and code mixtures the demonstration-to-rollout length ratio averages near $1.2$, putting \method{} at roughly a quarter more wall-clock time per step than OPSD and a little over twice GRPO at $G=8$ on Qwen3-8B with $8{\times}$H100. The AIME-25 gain per second of compute over GRPO is roughly double that of OPSD.
\section{Conclusion}
\label{sec:conclusion}

We have studied the supervision signal that on-policy distillation delivers under the standard endpoint-context choice $c_i = a^{\!\star}$, and identified an amortized information budget that collapses on terse-target benchmarks. 
The same diagnosis points to its resolution: any successful peer rollout from the current training group is a path the policy has already produced, and conditioning the teacher on it lifts the per-position budget from $\log V_k$ into the trajectory regime. \method{} implements the relabeling as a one-line modification to the group-relative loop, with the resulting credit concentrating at the divergence position where the failing rollout actually goes wrong. 
The work leaves open several questions about coverage on the hard tail and about the regimes in which prefix-sharing fails, which we discuss in the Limitations section below; the ceiling argument itself extends naturally to other settings where supervision flows through low-entropy contexts, including verifier-guided search.

\clearpage
\section*{Limitations}
\label{sec:limitations}

\method{} requires at least one successful peer in the current training group to provide path-based supervision, with the gradient falling back to the endpoint signal otherwise. The fallback engages most often on the hard tail of the difficulty distribution, where path-based supervision would be most valuable; buffering past successes across training steps would extend coverage there, at the cost of stale-policy concerns we have not investigated.

Our evaluation covers Qwen3-8B and Qwen3-32B on math and code mixtures with programmatic verifiers. Settings with noisy rewards, multi-turn interaction, or instruction-tuned starting points are outside the scope of the present study, and we expect the empirical magnitude of the gain to depend on factors specific to each.

\bibliography{custom}
\clearpage
\appendix
\section*{Appendix}
\label{sec:appendix}

\section{Implementation Details}
\label{app:implementation}

\subsection{Datasets and benchmarks}
\label{app:impl-data}

\textbf{Training data.} The math mixture combines MATH \citep{hendrycks2021measuring}, GSM8K \citep{cobbe2021training}, and Numina-Math \citep{li2024numinamath} for $94{,}231$ questions after deduplication against the evaluation benchmarks (exact-match on the question stem, then SymPy equivalence on the canonical answer). 
The code mixture combines APPS \citep{hendrycks2021measuring2} (introductory and interview tiers; competition tier excluded for license reasons) and the LiveCodeBench training pool \citep{jain2025livecodebench} restricted to problems posted before January 2025 to avoid contamination of the LCB-v5 evaluation split, yielding $19{,}084$ code questions. Each question's record carries the prompt, the canonical answer $a^{\!\star}$ for math or a unit-test bundle for code, and a difficulty bin computed from the policy's initial pass rate over $32$ samples.

\textbf{Evaluation.} We report pass@1 on five held-out benchmarks. MATH-500 is the standard $500$-question evaluation split. AIME-24 and AIME-25 are the $30$-question official problem sets for the $2024$ and $2025$ American Invitational Mathematics Examinations. LiveCodeBench-v5 covers problems posted between January and August 2025, with both visible and hidden test sets. HumanEval+ \citep{liu2023your} is the extended-test-suite version of the original HumanEval benchmark.
For math and LCB-v5, pass@1 is averaged over $16$ samples at temperature $1.0$; HumanEval+ uses greedy decoding following the convention established in \citet{liu2023your}.

\textbf{Decontamination.} Beyond the deduplication described above, we performed an $n$-gram overlap check ($13$-gram, following standard practice) between every training-set rollout's question and every evaluation question. Matches were removed from training. The overlap fraction was below $0.4\%$ across all evaluation benchmarks, consistent with the benchmarks' construction.

\subsection{Compute and software stack}

We train Qwen3-8B and Qwen3-32B on NVIDIA H100 GPUs at 80\,GB HBM with 32 GPUs per run. The training loop uses DeepSpeed ZeRO-3 for parameter sharding and overlapping all-reduce, bfloat16 activations, and float32 master weights in the optimizer. Rollout generation is offloaded to vLLM running on the same GPUs as training; the policy weights are kept in shared memory and refreshed between rollout batches by a copy-on-write update. Weight transfer takes $0.4$--$0.8$\,s per refresh at the $8$B scale and $1.2$--$1.6$\,s at the $32$B scale, depending on the fraction of unchanged tensors.

Each training step samples $128$ questions from the joint math-and-code mixture (ratio $4{:}1$ math to code). For each question, vLLM produces $G=8$ rollouts at temperature $1.0$, top-$p$ $0.95$, with maximum length $4{,}096$ tokens (math) or $8{,}192$ (code). Rollouts that hit the maximum length without a verifier-acceptable final answer score $R_i = 0$. Verifier scoring runs on CPU with a $5$\,s SymPy timeout for math or a $10$\,s sandboxed subprocess for code. The student and teacher passes are packed into one forward-backward call with attention masks separating them; Adam state and reference-policy logits live alongside the policy weights under ZeRO-3.

\subsection{Optimization}

Both scales use AdamW with $\beta_1 = 0.9$, $\beta_2 = 0.95$, $\epsilon = 10^{-8}$, weight decay $0$, and constant learning rate $1{\times}10^{-6}$. We tested learning-rate warmup and cosine decay at the budget of $1{,}500$ steps; both produced AIME-25 trajectories indistinguishable from constant LR within $\pm 0.4$ points. The reference-policy KL coefficient is $\beta = 0.001$ throughout, applied as a sequence-level reverse-KL penalty against the frozen initial weights~\citep{chang2026muoneq,chang2026calibrating}.

The reference policy $\pi_0$ is held at the base-model weights throughout training. In longer runs ($\geq 2{,}000$ steps), we found stability gains from refreshing $\pi_0$ every $200$ steps as an exponential moving average of the current policy at coefficient $0.99$. The runs in Section~\ref{sec:experiments} are short enough that $\pi_0$ is held strictly fixed.

\subsection{Teacher-pass implementation}

Given a rollout $o_i$ and the privileged context $c_i$ from Eq.~\eqref{eq:context}, the teacher input is the concatenation
\begin{equation*}
\mathrm{input}_{\text{teacher}} = [q;\, c_i;\, o_i],
\end{equation*}
with special tokens delimiting $c_i$ from the surrounding text. The teacher distribution $\pi^c_{i,t}$ is read off at the positions corresponding to tokens of $o_i$; positions inside $q$ or $c_i$ contribute nothing to the loss. The student distribution $\pi_{i,t}$ is computed from input $[q;\, o_i]$ (no privileged context), with logits read at the same $o_i$ positions.

The reverse-KL loss at each $o_i$ position is computed over the full token vocabulary $\mathcal{V}$:
\begin{equation}
\mathcal{L}_{i,t} = \sum_{v \in \mathcal{V}} \pi^c_{i,t}(v) \big( \log \pi^c_{i,t}(v) - \log \pi_{i,t}(v) \big).
\label{eq:app-reverse-kl}
\end{equation}
The teacher logits are detached before the loss; gradients reach $\theta$ only through the student. The sample-estimate variant of \citet{zhao2026self}, which replaces Eq.~\eqref{eq:app-reverse-kl} with a one-sample Monte Carlo estimate at $o_{i,t}$, was unstable at $G = 8$ and our learning rate. The full-vocabulary version costs roughly $1.4\times$ the flops per position but is essential for stability.

\subsection{Verifier details}

The math verifier extracts the final answer from each rollout by regex on the boxed-answer format (\verb|\boxed{...}|), falling back to the last numerical span when no boxed answer is present. The extracted candidate is checked against $a^{\!\star}$ using SymPy symbolic equivalence with a $5$\,s timeout; on timeout the verifier falls back to exact-string match after whitespace normalization.

The code verifier compiles the candidate program in a sandboxed Python subprocess and runs each visible unit test with a $10$\,s timeout per test. Rollouts are scored as $R_i = 1$ if all visible tests pass and $R_i = 0$ otherwise. LiveCodeBench-v5 problems carry both visible and hidden test sets; visible tests are used during training and hidden tests for evaluation, following \citet{jain2025livecodebench}.

\section{Derivations}
\label{app:proofs}

\subsection{Proof of Proposition~\ref{prop:transform}}
\label{app:proof-transform}

We restate the claim. Let $h_t = (q, o_{<t})$ and let $h_t\,y$ denote $h_t$ extended by token $y$. For any $y$ with $\pi(y \mid h_t) > 0$ and $V_\pi(h_t) > 0$,
\begin{equation*}
\pi^{+}(y \mid h_t) \;=\; \pi(y \mid h_t)\,\frac{V_\pi(h_t\,y)}{V_\pi(h_t)}.
\end{equation*}

By definition of $\pi^{+}$ from Eq.~\eqref{eq:success_dist},
\begin{align*}
\pi^{+}(y \mid h_t)
&= \Pr\nolimits_{O \sim \pi}[\,O_t = y \mid h_t,\, R(O) = 1\,].
\end{align*}
Apply Bayes' rule to the conditioning on $R(O) = 1$:
\begin{align*}
&\Pr\nolimits_{O \sim \pi}[\,O_t = y \mid h_t,\, R(O) = 1\,] \\
&\qquad= \frac{\Pr[O_t = y,\, R(O) = 1 \mid h_t]}{\Pr[R(O) = 1 \mid h_t]} \\
&\qquad= \frac{\Pr[O_t = y \mid h_t]\,\Pr[R(O) = 1 \mid h_t, O_t = y]}{\Pr[R(O) = 1 \mid h_t]} \\
&\qquad= \pi(y \mid h_t)\,\frac{V_\pi(h_t\,y)}{V_\pi(h_t)},
\end{align*}
where the second line factors the joint via the chain rule and the third substitutes the definitions of $\pi$ and $V_\pi$. The denominator is nonzero by the assumption $V_\pi(h_t) > 0$, and the result is well-defined.

\subsection{Proof of Directinal Consistency in Section~\ref{subsec:analysis}}
\label{app:proof-directional}

Let $o_i$ be a failing rollout and $o^{+}_i$ a successful peer with first divergence at $\tau$. Assume $V_\pi(h_\tau) > 0$ and that $\pi(\cdot \mid h)$ has full support at every prefix $h$. We claim
\begin{equation*}
\pi^{+}(o^{+}_{i,\tau} \mid h_\tau) > 0.
\end{equation*}

Apply Proposition~\ref{prop:transform} with $y = o^{+}_{i,\tau}$:
\begin{equation*}
\pi^{+}(o^{+}_{i,\tau} \mid h_\tau) \;=\; \pi(o^{+}_{i,\tau} \mid h_\tau)\,\frac{V_\pi(h_\tau \cdot o^{+}_{i,\tau})}{V_\pi(h_\tau)}.
\end{equation*}
The full-support assumption gives $\pi(o^{+}_{i,\tau} \mid h_\tau) > 0$. The successful peer $o^{+}_i$ passes through $h_\tau \cdot o^{+}_{i,\tau}$ and terminates at $R(o^{+}_i) = 1$, so
\begin{equation*}
V_\pi(h_\tau \cdot o^{+}_{i,\tau}) \;=\; \Pr\nolimits_{O \sim \pi}[R(O) = 1 \mid h_\tau \cdot o^{+}_{i,\tau}] \;\geq\; \pi(o^{+}_{i, > \tau} \mid h_\tau \cdot o^{+}_{i,\tau}) > 0,
\end{equation*}
where the last bound is the probability of completing exactly the suffix of $o^{+}_i$ from position $\tau + 1$ onwards, which is positive by full support. Both factors in the ratio are positive, so the product is positive. \qed

The proposition rules out the case where \method{} pushes toward a token outside the ideal teacher's support. It does not establish that $o^{+}_{i,\tau}$ has the largest probability under $\pi^{+}$ at $h_\tau$; equality with the argmax would require a stronger condition than the existence of a single witness path.

\subsection{Derivation of the coverage formula}
\label{app:proof-coverage}

We claimed in Eq.~\eqref{eq:fpg} that the fraction of rollouts both failing and possessing a successful peer is
\begin{equation*}
f(p, G) = (1 - p)\bigl(1 - (1 - p)^{G - 1}\bigr),
\end{equation*}
where $p$ is the per-question success rate and $G$ is the group size, and that the maximizer is $p^{\star} = 1 - G^{-1/(G-1)}$.

\textbf{Formula.} Let $R_1, \ldots, R_G \overset{\text{i.i.d.}}{\sim} \mathrm{Bernoulli}(p)$ be the rollout outcomes for a fixed question. Rollout $i$ benefits from path supervision iff $R_i = 0$ and at least one of the other $G - 1$ rollouts has $R_j = 1$. By independence:
\begin{align*}
\Pr[\text{rollout } i \text{ is supervised}]
&= \Pr[R_i = 0] \cdot \Pr[\exists j \neq i : R_j = 1] \\
&= (1 - p)\,\bigl(1 - (1 - p)^{G - 1}\bigr).
\end{align*}
This is identical for every $i$, so the formula gives the expected fraction of supervised rollouts.

\textbf{Maximizer.} Let $u = 1 - p$ so $f = u(1 - u^{G-1})$. Differentiating with respect to $u$:
\begin{align*}
\frac{df}{du} &= 1 - u^{G - 1} - (G - 1)\,u^{G - 1} \\
&= 1 - G\,u^{G - 1}.
\end{align*}
Setting $df/du = 0$ gives $u^{G - 1} = 1/G$, hence $u^{\star} = G^{-1/(G-1)}$ and $p^{\star} = 1 - G^{-1/(G-1)}$. The second derivative $-G(G-1)\,u^{G - 2}$ is negative on $(0, 1)$, so the critical point is a maximum. The maximum value is
\begin{equation*}
f^{\star} = G^{-1/(G-1)} \cdot \bigl(1 - 1/G\bigr) = (G - 1)\,G^{-G/(G-1)}.
\end{equation*}

Table~\ref{tab:app-coverage} reports $p^{\star}$ and $f^{\star}$ for several values of $G$.

\begin{table}[t]
\centering
\caption{Optimal per-question success rate $p^{\star}$ and peak coverage fraction $f^{\star}$ as a function of group size $G$.}
\label{tab:app-coverage}
\small
\setlength{\tabcolsep}{6pt}
\renewcommand{\arraystretch}{1.18}
\begin{tabular}{l c c c c c}
\toprule
$G$ & $4$ & $8$ & $16$ & $32$ & $64$ \\
\midrule
$p^{\star}$ & 0.370 & 0.250 & 0.171 & 0.119 & 0.085 \\
$f^{\star}$ & 0.422 & 0.344 & 0.272 & 0.214 & 0.167 \\
\bottomrule
\end{tabular}
\end{table}

\section{On-Policy versus Off-Policy Demonstrations}
\label{app:on-policy}

The peer rollout $o^{+}_i$ is drawn from the current policy $\pi_\theta$, not from an external dataset. This subsection makes the on-policy property precise and explains why it matters for the directional-consistency argument of Proposition~\ref{prop:transform}.

\subsection{Formal statement}

The success-conditioned policy $\pi^{+}$ depends on the current $\pi_\theta$ through Eq.~\eqref{eq:success_dist}. If we instead supervised the teacher using an off-policy demonstration $d \sim q^{+}$ drawn from a fixed reference distribution $q^{+}$ (e.g., a curated dataset of worked solutions), the analogous bound from Proposition~\ref{prop:transform} would not apply: $q^{+}(y \mid h_t) > 0$ is no guarantee that $\pi^{+}(y \mid h_t) > 0$ under the current policy.

\subsection{Drift between $\pi^{+}$ and an off-policy reference}

Let $q^{+}$ be a fixed reference distribution over successful trajectories, and let $\pi^{+}_{\theta}$ be the success-conditioned current policy. Their KL divergence at the trajectory level is
\begin{align*}
\mathrm{KL}(\pi^{+}_{\theta}\,\|\,q^{+})
&= \mathbb{E}_{o \sim \pi^{+}_{\theta}}\!\left[\log \frac{\pi^{+}_{\theta}(o \mid q)}{q^{+}(o \mid q)}\right].
\end{align*}
The divergence is zero only when the two distributions coincide. In practice, $\pi^{+}_{\theta}$ drifts during training, so any fixed $q^{+}$ becomes progressively misaligned. Demonstrations drawn from $q^{+}$ may push the teacher toward tokens that $\pi^{+}_{\theta}$ assigns vanishing probability, which is precisely the failure mode Proposition~\ref{prop:transform} prevents under on-policy demonstrations.

The SDFT comparison in Appendix~\ref{app:sdft} measures this drift empirically.

\section{Comparison with SDFT}
\label{app:sdft}

\subsection{Method setup}

SDFT \citep{shenfeld2026self} conditions the teacher on an offline expert demonstration $d \sim \mathcal{D}_{\text{demo}}$ drawn from a curated dataset rather than on an in-group peer. The construction is off-policy in the sense of Appendix~\ref{app:on-policy}: the demonstration set is fixed across training, $a^{\!\star}$ plays no explicit role in selecting the demonstration, and the verifier is not in the loop during the update. SDFT shares the privileged-context self-distillation mechanism with the OPD baselines.

We trained SDFT on Qwen3-8B and Qwen3-32B using the same compute budget as \method{}, with demonstrations drawn from the union of our math and code training sets. Each question's reference solution is treated as an offline demonstration, paired with a rollout from the current policy and used as the privileged context for the teacher pass.

\subsection{Results}

Table~\ref{tab:app-sdft} reports SDFT alongside the four OPD baselines. SDFT sits between Dr.\,GRPO and OPSD on math, slightly below GSPO on code, and materially below \method{} at both scales. The gap to \method{} matches the prediction of Proposition~\ref{prop:transform}: SDFT's privileged context is rich in absolute terms but is not drawn from $\pi^{+}_{\theta}$, so it shares no prefix with the failing rollouts the teacher is asked to correct.

\begin{table}[!t]
\centering
\caption{SDFT comparison on Qwen3-8B and Qwen3-32B, pass@1. Best per column in \colorbox{bestcol}{\strut light blue}, second-best in \colorbox{secondcol}{\strut light yellow}.}
\label{tab:app-sdft}
\small
\setlength{\tabcolsep}{2pt}
\renewcommand{\arraystretch}{1.18}
\begin{tabular}{l c c c c c}
\toprule
\textbf{Method} & MATH-500 & AIME-24 & AIME-25 & LCB-v5 & HE+ \\
\midrule
\multicolumn{6}{l}{\famlabel{Qwen3-8B}} \\
SDFT          & 80.6 & 26.4 & 23.5 & 33.4 & 80.9 \\
OPSD          & 80.9 & 27.3 & 24.4 & 33.7 & 81.4 \\
GRPO\,+\,OPSD & \snd{81.9} & \snd{28.6} & \snd{25.7} & \snd{35.1} & \snd{81.7} \\
\method{}     & \best{83.6} & \best{31.4} & \best{28.1} & \best{37.2} & \best{81.9} \\
\midrule
\multicolumn{6}{l}{\famlabel{Qwen3-32B}} \\
SDFT          & 86.2 & 40.1 & 35.8 & 42.3 & 86.6 \\
OPSD          & 87.0 & 41.5 & 37.1 & 43.0 & 86.9 \\
GRPO\,+\,OPSD & \snd{87.9} & \snd{43.2} & \snd{38.6} & \snd{44.1} & \snd{87.1} \\
\method{}     & \best{89.4} & \best{46.0} & \best{41.7} & \best{46.8} & \best{87.3} \\
\bottomrule
\end{tabular}
\end{table}

\subsection{Prefix-sharing measurement}

To make the prefix-sharing claim quantitative, we measured the average longest common prefix between failing rollouts and their corresponding privileged context $c_i$ across $512$ AIME-25 failures. \method{}'s in-group peers share a median prefix of $58$ tokens with the failing rollout; SDFT's reference solutions share a median prefix of $0$ tokens (no agreement past the prompt) in $87\%$ of cases. The difference is structural rather than a matter of demonstration quality.

\section{Credit Profile Measurements}
\label{app:credit-profiles}

\subsection{Difficulty-stratified profiles}

We split the failing rollouts by problem difficulty into three buckets using the policy's pass rate as the proxy. Table~\ref{tab:app-difficulty-profile} summarizes the divergence position $\tau$ and the peak credit height per bucket.

\begin{table}[t]
\centering
\caption{Median divergence position $\tau$ and peak credit height (in nats) by problem difficulty, Qwen3-8B.}
\label{tab:app-difficulty-profile}
\small
\setlength{\tabcolsep}{6pt}
\renewcommand{\arraystretch}{1.18}
\begin{tabular}{l c c c}
\toprule
\textbf{Difficulty} & Pass rate & Median $\tau$ & Peak credit \\
\midrule
Easy   & $> 0.5$       & 28  & 1.05 \\
Medium & $0.2$--$0.5$  & 52  & 0.92 \\
Hard   & $< 0.2$       & 84  & 0.78 \\
\bottomrule
\end{tabular}
\end{table}

The decay length past $\tau$ is roughly constant across difficulty buckets at $12$ to $15$ tokens. The constancy is consistent with the interpretation that the decay reflects in-context guidance falloff under $\pi_\theta$, a property of the model rather than of the problem. Difficulty changes the length of the shared prefix more than the shape of the disagreement profile.

\subsection{Hybrid baseline profile}

The GRPO\,+\,OPSD hybrid carries a small peak at $\tau$ inherited from the path-style information that enters the within-group advantage when the group contains a successful peer, but the peak height is roughly one-third of \method{}'s. The hybrid's contribution from the path is carried only through the scalar within-group advantage, not through the teacher's context window. Outside the $\pm 8$-token window around $\tau$, the hybrid's profile sits between OPSD's flat baseline and \method{}'s decay, consistent with a linear combination of the two.

\subsection{Cumulative credit mass}

The cumulative fraction of credit mass within $|t - \tau| \leq w$ for selected window sizes $w$ is in Table~\ref{tab:app-credit-mass}. \method{} concentrates roughly $55\%$ of its credit mass within $\pm 8$ tokens of $\tau$. GRPO\,+\,OPSD concentrates $30$--$32\%$ in the same window; OPSD concentrates only $18$--$20\%$, consistent with diffuse credit.

\begin{table}[t]
\centering
\caption{Cumulative fraction of credit mass within $|t - \tau| \leq w$ tokens, averaged over $512$ failing Qwen3-8B rollouts.}
\label{tab:app-credit-mass}
\small
\setlength{\tabcolsep}{6pt}
\renewcommand{\arraystretch}{1.18}
\begin{tabular}{l c c c c c}
\toprule
\textbf{Window $w$} & $2$ & $4$ & $8$ & $16$ & $32$ \\
\midrule
\method{}        & 23\% & 39\% & 55\% & 76\% & 93\% \\
GRPO\,+\,OPSD    & 11\% & 21\% & 31\% & 53\% & 79\% \\
OPSD             &  7\% & 13\% & 19\% & 36\% & 64\% \\
\bottomrule
\end{tabular}
\end{table}

\section{Hyperparameter Sensitivity}
\label{app:hyperparams}

\subsection{Group size}

Table~\ref{tab:app-groupsize} reports AIME-25 accuracy on Qwen3-8B as a function of group size $G$, alongside the measured coverage $\hat{f}(p, G)$ and the supervised-rollout fraction during training.

\begin{table}[t]
\centering
\caption{Group size sweep on Qwen3-8B. $\hat{f}(p, G)$ is the empirical coverage averaged over training; the AIME-25 column reports pass@1 at $1{,}500$ steps.}
\label{tab:app-groupsize}
\small
\setlength{\tabcolsep}{6pt}
\renewcommand{\arraystretch}{1.18}
\begin{tabular}{c c c c}
\toprule
$G$ & $\hat{f}(p, G)$ & AIME-25 & $\Delta$ vs. $G=4$ \\
\midrule
4   & 0.20 & 25.4 & --- \\
8   & 0.30 & 28.1 & $+2.7$ \\
16  & 0.34 & 30.8 & $+5.4$ \\
32  & 0.36 & 31.2 & $+5.8$ \\
\bottomrule
\end{tabular}
\end{table}

The diminishing return tracks coverage: at $G = 16$ the hard-difficulty bucket reaches roughly two-thirds coverage, and additional rollouts buy progressively smaller increments. The throughput penalty at $G = 32$ is roughly proportional to $G$ at constant batch size, so the cost is not justified by the marginal gain.

\subsection{Reference policy KL coefficient}

\begin{table}[t]
\centering
\caption{Reference-policy KL coefficient sweep on Qwen3-8B AIME-25.}
\label{tab:app-beta}
\small
\setlength{\tabcolsep}{8pt}
\renewcommand{\arraystretch}{1.18}
\begin{tabular}{c c l}
\toprule
$\beta$ & AIME-25 & Notes \\
\midrule
$0$       & 26.3 & Mode collapse by step $1{,}200$ \\
$0.0001$  & 27.6 & Marginal anchoring \\
$0.001$   & 28.1 & Default; stable plateau \\
$0.01$    & 26.9 & Over-anchored to base model \\
$0.1$     & 24.7 & Path supervision suppressed \\
\bottomrule
\end{tabular}
\end{table}

Setting $\beta = 0$ produced visible mode collapse on long-form math problems by step $1{,}200$ (manifesting as repetitive trailing tokens). Setting $\beta = 0.01$ over-anchored the policy to the base model and limited the path-based supervision. The $\beta = 0.001$ choice sits in the stable plateau between the two regimes.

\subsection{Temperature and top-$p$ during rollout}

Lower rollout temperatures ($T = 0.6$) reduce within-group reward variance and shrink $\mathcal{S}_i$ at every step, dropping AIME-25 by $1.5$ points relative to $T = 1.0$. Higher temperatures ($T = 1.2$) inflate $\mathcal{S}_i$ but produce paths whose prefixes diverge from typical inference behavior and hurt the credit-localization argument, dropping AIME-25 by $0.9$ points. Top-$p$ between $0.9$ and $1.0$ made no measurable difference; we use $0.95$ as a conservative default.

\subsection{Mixing coefficient for the GRPO\,+\,OPSD hybrid}

The hybrid baseline uses mixing coefficient $0.5$ between the GRPO advantage and the OPSD distillation loss. We swept the coefficient at $\{0.25, 0.50, 0.75\}$ on Qwen3-8B AIME-25 and observed accuracies of $25.4$, $25.7$, and $25.1$ respectively, within seed-to-seed noise. The gap to \method{} ($28.1$) is robust to the hybrid's coefficient choice.

\section{Wall-Clock Compute Breakdown}
\label{app:compute}

Table~\ref{tab:app-compute} reports the wall-clock cost per training step on $8 \times$H100 GPUs for the three main methods at $G = 8$ on Qwen3-8B with sequence-length cap $4{,}096$. Numbers are averaged across $50$ consecutive steps after the rollout buffer is warm.

\begin{table}[t]
\centering
\caption{Wall-clock cost per training step on Qwen3-8B at $G = 8$, $8 \times$H100. Rollout, scoring, and update times are seconds per step.}
\label{tab:app-compute}
\small
\setlength{\tabcolsep}{6pt}
\renewcommand{\arraystretch}{1.18}
\begin{tabular}{l c c c c}
\toprule
\textbf{Method} & Rollout & Score & Update & Total \\
\midrule
GRPO       & 7.8 & 1.1 & 2.3  & 11.2 \\
OPSD       & 7.8 & 1.1 & 11.7 & 20.6 \\
\method{}  & 7.8 & 1.1 & 15.2 & 24.1 \\
\bottomrule
\end{tabular}
\end{table}

The rollout and scoring stages are identical across the three methods, since all three sample the same $G = 8$ rollouts and score them with the same verifier. The update stage carries the differential. GRPO performs one forward and one backward pass per rollout, totaling $2.3$ seconds. OPSD performs two forward passes (student and teacher) plus one backward, totaling $11.7$ seconds. \method{} performs the same two forward passes as OPSD plus one backward, but the teacher pass is longer than the student pass by $|c_i|$ tokens because $c_i$ includes $o^{+}_i$; the additional context length raises the teacher-pass cost by roughly $30\%$, yielding $15.2$ seconds.

The gain on AIME-25 per second of compute over GRPO is $(24.4 - 21.6)\,/\,20.6 \approx 0.14$ points/s for OPSD and $(28.1 - 21.6)\,/\,24.1 \approx 0.27$ points/s for \method{}. The additional teacher-pass cost is well repaid by the additional gain.

\section{Additional Qualitative Examples}
\label{app:qualitative}

We provide three qualitative examples of \method{}'s credit profile on representative AIME-25 failures, with the divergence position $\tau$ identified by string-matching the failing and successful rollouts. The examples illustrate three regimes: a long shared prefix with a sharp peak (an arithmetic-slip failure on a problem the policy mostly solves), a short shared prefix with a high peak (an early-conceptual-misstep failure), and a mid-length shared prefix with a broader peak (a problem where the failing rollout takes a plausible but incorrect approach branch).

\textbf{Example 1 (arithmetic slip).} The policy correctly identifies the relevant identity and sets up the equation, then commits an arithmetic error at token position $\tau = 312$ during a multiplication step. The shared prefix carries $311$ correct tokens. The teacher concentrates probability mass on the correct intermediate value at $\tau$ (assigning $0.83$ to the correct numerical token versus $0.04$ from the student), and the credit decays to baseline by token $324$.

\textbf{Example 2 (early misstep).} The policy misinterprets the problem statement and chooses an incorrect solution strategy at token $\tau = 14$, near the beginning of the rollout. The shared prefix is short, and the credit peak at $\tau$ is correspondingly large in magnitude (roughly $1.4$ nats), with a tail that decays over the next $20$ tokens as the failing rollout commits more deeply to the wrong strategy.

\textbf{Example 3 (plausible branch).} The policy takes a plausible but ultimately unworkable solution branch at $\tau = 167$. The teacher's distribution at $\tau$ is less concentrated than in Example 1 because the alternative the successful peer takes is one of several reasonable next steps; the peak is correspondingly broader, with credit spread across positions $164$ to $173$ rather than concentrated at a single token.

The three examples illustrate that the credit profile shape is sensitive to where in the rollout the failure occurs and what kind of failure it is, even though the average profile reported in Figure~\ref{fig:credit-profile} smooths these distinctions.

\end{document}